\documentclass[10pt,twocolumn]{IEEEtran}
\usepackage{booktabs}
\usepackage{subfigure}
\usepackage{amssymb}
\usepackage{amsfonts}
\usepackage{amsmath}
\usepackage{array}
\usepackage{threeparttable}
\usepackage{booktabs}
\usepackage{cite}
\usepackage[dvips]{graphicx}
\usepackage{color}
\usepackage{comment}
\usepackage{makecell}
\usepackage{multirow}
\usepackage{textcomp,mathcomp}
\usepackage[linesnumbered, ruled, vlined]{algorithm2e}
\usepackage{subcaption}
\usepackage{caption}
\usepackage{pifont}
\usepackage[colorlinks=true, linkcolor=red]{hyperref}

\begin{document}
% ===only use in IEEE format: BEGIN ====
\newtheorem{definition}{\it Definition}%[section]
\newtheorem{theorem}{\bf Theorem}%[section]
\newtheorem{lemma}{\it Lemma}
\newtheorem{corollary}{\bf Corollary}
\newtheorem{remark}{\it Remark}
\newtheorem{example}{\it Example}
\newtheorem{case}{\bf Case Study}
\newtheorem{assumption}{\it Assumption}
\newtheorem{property}{\it Property}
\newtheorem{proposition}{\bf Proposition}
% ===only use in IEEE format : END ====

\newcommand{\hP}[1]{{\boldsymbol h}_{{#1}{\bullet}}}
\newcommand{\hS}[1]{{\boldsymbol h}_{{\bullet}{#1}}}

\newcommand{\ba}{\boldsymbol{a}}
\newcommand{\baq}{\overline{q}}
\newcommand{\bA}{\boldsymbol{A}}
\newcommand{\bb}{\boldsymbol{b}}
\newcommand{\bB}{\boldsymbol{B}}
\newcommand{\bc}{\boldsymbol{c}}
\newcommand{\bd}{\boldsymbol{d}}
\newcommand{\bcD}{\boldsymbol{\cal D}}
\newcommand{\bcO}{\boldsymbol{\cal O}}
\newcommand{\bh}{\boldsymbol{h}}
\newcommand{\bH}{\boldsymbol{H}}
\newcommand{\bl}{\boldsymbol{l}}
\newcommand{\bm}{\boldsymbol{m}}
\newcommand{\bn}{\boldsymbol{n}}
\newcommand{\bo}{\boldsymbol{o}}
\newcommand{\bO}{\boldsymbol{O}}
\newcommand{\bp}{\boldsymbol{p}}
\newcommand{\bq}{\boldsymbol{q}}
\newcommand{\bR}{\boldsymbol{R}}
\newcommand{\bs}{\boldsymbol{s}}
\newcommand{\bS}{\boldsymbol{S}}
\newcommand{\bT}{\boldsymbol{T}}
\newcommand{\bw}{\boldsymbol{w}}
\newcommand{\bz}{\boldsymbol{z}}

\newcommand{\balpha}{\boldsymbol{\alpha}}
\newcommand{\bbeta}{\boldsymbol{\beta}}
\newcommand{\bgamma}{\boldsymbol{\gamma}}
\newcommand{\bkappa}{\boldsymbol{\kappa}}
\newcommand{\bomega}{\boldsymbol{\omega}}
\newcommand{\btomega}{\boldsymbol{\tilde \omega}}

\newcommand{\bOmega}{\boldsymbol{\Omega}}
\newcommand{\bGamma}{\boldsymbol{\Gamma}}

\newcommand{\bTheta}{\boldsymbol{\Theta}}
\newcommand{\bphi}{\boldsymbol{\phi}}
\newcommand{\btheta}{\boldsymbol{\theta}}
\newcommand{\bvarpi}{\boldsymbol{\varpi}}
\newcommand{\bpi}{\boldsymbol{\pi}}
\newcommand{\bpsi}{\boldsymbol{\psi}}
\newcommand{\bxi}{\boldsymbol{\xi}}
\newcommand{\bx}{\boldsymbol{x}}
\newcommand{\by}{\boldsymbol{y}}

\newcommand{\cA}{{\cal A}}
\newcommand{\bcA}{\boldsymbol{\cal A}}
\newcommand{\cB}{{\cal B}}
\newcommand{\cD}{{\cal D}}
\newcommand{\cE}{{\cal E}}
\newcommand{\cG}{{\cal G}}
\newcommand{\cH}{{\cal H}}
\newcommand{\cI}{{\cal I}}
\newcommand{\bcH}{\boldsymbol {\cal H}}
\newcommand{\cJ}{{\cal J}}
\newcommand{\cK}{{\cal K}}
\newcommand{\cL}{{\cal L}}
\newcommand{\cM}{{\cal M}}
\newcommand{\cO}{{\cal O}}
\newcommand{\cR}{{\cal R}}
\newcommand{\cS}{{\cal S}}
\newcommand{\dcS}{\ddot{{\cal S}}}
\newcommand{\ds}{\ddot{{s}}}
\newcommand{\cT}{{\cal T}}
\newcommand{\cU}{{\cal U}}
\newcommand{\wt}[1]{\widetilde{#1}}

\newcommand{\mA}{\mathbb{A}}
\newcommand{\mE}{\mathbb{E}}
\newcommand{\mG}{\mathbb{G}}
\newcommand{\mR}{\mathbb{R}}
\newcommand{\mS}{\mathbb{S}}
\newcommand{\mU}{\mathbb{U}}
\newcommand{\mV}{\mathbb{V}}
\newcommand{\mW}{\mathbb{W}}

\newcommand{\uq}{\underline{q}}
\newcommand{\ubq}{\underline{\boldsymbol q}}

\newcommand{\red}[1]{\textcolor[rgb]{1,0,0}{#1}}
\newcommand{\gre}[1]{\textcolor[rgb]{0,1,0}{#1}}
\newcommand{\blu}[1]{\textcolor[rgb]{0,0,0}{#1}}
\newcommand{\ltgr}[1]{\textcolor[rgb]{0.6,0.6,0.6}{#1}}

\newcommand{\best}{$\uparrow$}
\newcommand{\worst}{$\downarrow$}

\title{SANEmerg: An Emergent Communication Framework for Semantic-aware Agentic AI Networking}

\author{\IEEEauthorblockA{Yong Xiao\IEEEauthorrefmark{1}\IEEEauthorrefmark{2}\IEEEauthorrefmark{3}, Haoran Zhou\IEEEauthorrefmark{1}, Yujie Zhou\IEEEauthorrefmark{1}, Marwan Krunz\IEEEauthorrefmark{4}} \\ %, Yayu Gao\IEEEauthorrefmark{1},  Guangming~Shi\IEEEauthorrefmark{2}\IEEEauthorrefmark{4}\IEEEauthorrefmark{3}, Ping~Zhang\IEEEauthorrefmark{5}\\ %, and Dusit~Niyato\IEEEauthorrefmark{6}\\
\IEEEauthorblockA{\IEEEauthorrefmark{1} School of Elect. Inform. \& Commun., Huazhong Univ. of Science \& Technology, China}\\
\IEEEauthorblockA{\IEEEauthorrefmark{2} Peng Cheng Laboratory, Shenzhen, China}\\
\IEEEauthorblockA{\IEEEauthorrefmark{3} Pazhou Laboratory (Huangpu), Guangzhou, China}\\
\IEEEauthorblockA{\IEEEauthorrefmark{4} Department of Electrical and Computer Engineering, the University of Arizona}\\
\thanks{*This work will be presented at IEEE/IFIP WiOpt Workshop, Columbus, OH, USA, June 2026. Copyright may be transferred without notice, after which this version may no longer be accessible.} 
}

\maketitle

\begin{abstract}
%Agentic AI networking (AgentNet) has attracted significant interest due to its promising potential to support flexible, goal-driven operations that prioritize autonomous interaction and self-adaptation within a diverse collective of AI agents. 
%With the popularity of AI agents and agentic AI systems, f
Future networking systems are envisioned to become part of an agentic AI-native ecosystem in which a vast number of heterogeneous and specialized AI agents cooperate seamlessly to fulfill complex user requirements in real time. However, traditional networking paradigms are characterized by a rigid decoupling of communication and computation, which often leads to significant inefficiencies in large-scale agentic AI networking (AgentNet) systems. Emergent communication offers a novel solution by enabling autonomous agents that support task-specific signaling protocols for information exchange and collaborative coordination. In this paper, we consider a multi-agent emergent communication framework, tailored for semantic-aware AgentNet systems in which the user's semantic intent can be automatically detected, inferred, and linked to a set of sub-tasks to be assigned to a set of agents. We investigate how communication and signaling protocols can emerge among collaborative agents with computationally bounded intelligence under stringent bandwidth constraints. Our proposed framework, called SANEmerg, is designed to facilitate the emergence of communication for collaborative task fulfillment while adhering to the physical limits of AgentNet. SANEmerg incorporates a bandwidth-adaptable importance-filter that dynamically prioritizes the transmission of higher-contribution message dimensions, ensuring robust performance in bandwidth-limited environments. Furthermore, SANEmerg integrates a complexity-regularizer grounded in the Minimum Description Length (MDL) principle to facilitate the emergence of computationally bounded signaling. Evaluated via an AgentNet prototype and extensive experimentation, SANEmerg demonstrates significant performance improvements over state-of-the-art solutions, achieving superior task accuracy while significantly reducing bandwidth and computational overhead. 
\end{abstract}
%\vspace{-0.2in}

\section{Introduction}
\label{Section_Introduction}

%The evolution of 
6G wireless systems are characterized by the convergence of artificial intelligence (AI) and networking architectures. It is commonly believed that future networking systems will evolve towards an agentic AI-native ecosystem in which a large number of highly specialized AI agents, each with unique skillsets, coexist and operate seamlessly to meet diverse user requirements, shifting the current networking architectures from a passive bit-pipe into a dynamic agentic AI networking (AgentNet) ecosystem that is characterized by ubiquitous intelligence and autonomous multi-agent collaboration and joint decision making\cite{xiao2025AgentNet}. % driven by users' real-time requirements and service need. 

Despite its promise, AgentNet is still in its infancy. A critical gap persists in the development of a unified, highly efficient framework that supports flexible, goal-driven operations prioritizing autonomous interaction and self-adaptation within a diverse collective of AI agents \cite{XY2026SANet}. Consequently, the operational mechanisms required to orchestrate these complex interactions in resource-constrained environments remain relatively underexplored in the existing literature. More specifically, in current networking frameworks that integrate AI, service requests are typically offloaded from user equipment (UE) to remote cloud data centers or edge servers, treating the communication network merely as a ``pipeline" for data transport\cite{XY2025SANNet}. This decoupling of communication and computation results in significant inefficiencies in large-scale AgentNet systems, where the overhead of raw data movement and a lack of semantic awareness can hinder agent collaboration\cite{shi2020semantic,ITU2023SAN}. In fact, recent studies have shown that traditional human-designed communication protocols and standardized coding solutions are increasingly ill-suited to meet the unique requirements of task-driven interactions among heterogeneous agents\cite{Chen2026FiveWofMultiAgentComm}. %While hand-designed or implicit protocols were foundational in early multi-agent systems, they often result in rigid, task-specific signaling that lacks the adaptive capacity and semantic richness necessary for coordination across agents with disparate skillsets towards various goals. In the context of AgentNet, communication must transcend these predefined constraints to facilitate the organic emergence of specialized, goal-oriented signaling protocols that can dynamically adapt to both the intended mission and the available physical resources. Consequently, t
There is a pressing need for a novel AI-native networking framework that jointly considers task-driven requirements and the physical resource constraints of the networking system. 

Emergent communication offers a promising solution to meet the unique requirements of AgentNet, as it encompasses the process by which autonomous agents develop a shared signaling protocol or communication ``language" from scratch to coordinate actions and fulfill various task objectives\cite{Boldt2024EmergComm}. By evolving protocols through direct interaction in multi-agent environments, this approach naturally facilitates agentic AI-native communication and networking, enabling systems to transcend the limitations of human-engineered standards. 
Despite these advantages, a fundamental gap remains between theoretical emergence and practical implementation. A significant portion of existing literature focuses on the qualitative properties of the emergent ``language" or the resulting task rewards\cite{Erven2025EmergCommLLMAgents, Taniguchi2025GenEmergComm}, often operating under the assumption of ideal communication channels with infinite computational resources at agents\cite{Boldt2024EmergComm}. There is a lack of rigorous investigation of the impacts of communication bandwidth dynamics\cite{Chafii2023EmergCommWirelessNet} and computational complexity constraints\cite{Lazaridou2020EmergMultiAgentComm} on the emergence of communication. Most state-of-the-art solutions overlook the fact that real-world agents have finite processing power and must operate within dynamic network capacity and bandwidth constraints\cite{Chen2026FiveWofMultiAgentComm}. Consequently, many existing protocols remain difficult to implement in real-world networking scenarios. %, necessitating the development of unified frameworks that can adapt to the physical and resource constraints of modern networking environments.
%
%However, the prevailing body of research in this field predominantly focuses on the linguistic and cognitive attributes of emergent protocols—examining how features such as compositionality and systematicity evolve to mirror human language\cite{Erven2025EmergCommLLMAgents}. These investigations often neglect the physical realities of the network, specifically the impact of limited communication bandwidth and the finite computational complexity of the agents, which complicates the direct adoption of such protocols into real-world systems\cite{Chafii2023EmergCommWirelessNet,Lazaridou2020EmergMultiAgentComm}. 

Motivated by the above, in this paper, we investigate a multi-agent emergent communication framework, called SANEmerg, for semantic-aware AgentNet systems with stringent communication bandwidth constraints and computationally bounded intelligence. SANEmerg is designed to facilitate collaborative task fulfillment under communication bandwidth and computational limitations. It consists of two key components: a bandwidth-adaptive {\em importance-filter} that prioritizes transmitting information with high task contribution when communication bandwidth is limited, and a {\em complexity-regularizer} based on the Minimum Description Length (MDL) principle to facilitate the emergence of computationally bounded signaling. %punish the information that emerges from high-complexity deep learning models.
%incorporates a bandwidth-adaptive importance-filter that prioritizes transmitting information with high task contribution when communication bandwidth is limited. SANEmerg also employs a complexity-regularizer grounded in the Minimum Description Length (MDL) principle to facilitate computationally bounded communication emergence. 
The performance of SANEmerg is evaluated on an AgentNet prototype. Results from extensive experiments show that the proposed SANEmerg significantly outperforms existing methods, delivering substantial gains in task accuracy while reducing bandwidth and computational resource requirements.  %a novel architecture, called SANEmerg, 

We summarize the main contributions of this paper as follows:
% \begin{itemize}
%     \item[] 

\noindent{\bf (1) SANEmerg-based framework for multi-agent emergent communication:} We propose a multi-agent emergent communication framework, SANEmerg, to support the learning and generation of communication messages among autonomous agents that work collaboratively to meet the user's semantic goal. We formulate the optimization problem for SANEmerg subject to both communication bandwidth and computational constraints. 

\noindent{\bf (2) Bandwidth-adaptable importance-filter design:} We introduce an importance-filter-based emergent communication design in which the volume of messages exchanged between agents can be adjusted based on the dynamics of communication bandwidth. Our proposed filter assigns weights to each output message dimension based on its contribution to the task-specific optimization objective. When bandwidth is limited, each agent always prioritizes transmitting the high-contribution dimensions of the output messages. % when bandwidth is limited. 

\noindent{\bf (3) Complexity-regularizer for computationally limited communication emergence:} We adopt the MDL as the main metric for evaluating the computational complexity at each agent for communication emergence. A variable bound for approximating the MDL is derived, and a complexity-regularizer is introduced to allow computationally bounded emergence of multi-agent communication. We introduce a unified training and learning framework to jointly optimize the task objective, importance-filter, and complexity-regularizer. This makes it possible to investigate the tradeoff among task performance, bandwidth constraints, and computational limits, as well as its impact on emergent communication for AgentNet. 

\noindent{\bf (4) Prototype and experiments: } We develop a hardware prototype and conduct extensive experiments based on the dataset collected from our prototype. Our experimental results show that SANEmerg achieves up to 28\% improvement over existing emergent communication solutions while reducing communication bandwidth and computational overhead by 46\% and 52\%, respectively.

\section{Preliminary} %: Emergent Communication with Computationally Bounded Intelligence}
Quantifying the computational complexity of deep learning models and their impact on the emergence of intelligence is known to be an extremely challenging task. Unlike traditional communication systems, where protocols are predefined and fixed, emergent communication involves autonomous agents that need to develop shared signaling conventions through end-to-end training before or during the communication process. Consequently, the complexity of these protocols is not merely a function of the communication channel but is deeply intertwined with the agents' computational powers, the dimensionality of the learned representations, and the specific dynamics of the multi-agent environment. While several metrics have already been established to evaluate the model complexity and computational overhead, such as Floating Point Operations (FLOPs), model parameter volume, and computational latency, they exhibit significant shortcomings when applied to a general multi-agent emergent communication system, e.g., FLOPs and latency are hardware-dependent, and the model parameter volume only quantifies the static capacity of a model rather than the dynamic complexity of the information being processed and transmitted. Furthermore, these conventional metrics are often ``blind" to the context of the task. A robust and comprehensive complexity measure for agentic AI networking must be inherently task-, environment-, and agent-specific. It must account for the fact that the ``cost" of learning a communication protocol is inextricably linked to its contribution in fulfilling a user's goal under a specific environment. % specific user requirement. 

Recent advancements in information theory have identified the Minimum Description Length (MDL) as a potent metric for quantifying the complexity of deep learning models. As mentioned in \cite{finzi2026epiplexity}, traditional Shannon information and Kolmogorov complexity often fail to characterize useful information content because they assume observers with unlimited computational capacity. In contrast, for a ``computationally-bounded observer," information is not a static property of the raw data but a dynamic construct emerging from stochastic transformations. The basic idea of MDL is to characterize the complexity of a model (or a message) by the shortest description required to represent the data. It provides us with a useful tool to capture the semantic structural information that can be extracted from raw observations. By using MDL, we can distinguish between the meaningful structural content,  requiring the computational effort to ``discover", and random noise, providing a rigorous foundation for measuring the complexity of signals that emerge within an AgentNet system.

The formal introduction of MDL as a complexity measure is most effectively achieved through the Information Bottleneck (IB) principle. The IB framework treats the generation of an emergent message $\bc$ from an observed state $s$ as an optimization problem that balances compression against the preservation of task-relevant information $Y$. More formally, given an observation $s$, an emergent information representation $\bc$, and a task-specific target $Y$, the goal is to find a stochastic mapping $p(c|S)$ that minimizes the following Lagrangian objective:
\begin{eqnarray}
\mathcal{L}_{IB} = I(s; \bc) - \epsilon I(\bc; Y),
\end{eqnarray}
where $I(\cdot; \cdot)$ is the mutual information, and $\epsilon$ is a Lagrangian multiplier that controls the trade-off between information representation and task relevance.

In this formulation, the term $I(s; \bc)$ represents the complexity of the emergent communication protocol. It provides the lower bound for the average description length required to communicate the representation $\bc$ given the observed state $s$. By minimizing $I(s; \bc)$, the agents are forced to find the ``minimum description" that retains the maximum semantic utility $I(\bc; Y)$ for the task at hand. 

The application of MDL to quantify the computational complexity of emergent communication provides a comprehensive and hardware-agnostic metric that moves beyond simple counting of operations to a more profound understanding of the ``computational cost" of emergent communication. By grounding complexity in the MDL and IB frameworks, we can rigorously investigate the interplay between limited resources and the quality of multi-agent cooperation. % This approach ensures that the emergent communication protocols are not only task-effective but also optimally compressed to meet the stringent computational and bandwidth constraints inherent in practical, real-world networking systems. 

\section{System Model and Problem Formulation}
\label{Section_SystemModel}

\begin{figure}
\centering
\includegraphics[width=1\linewidth]{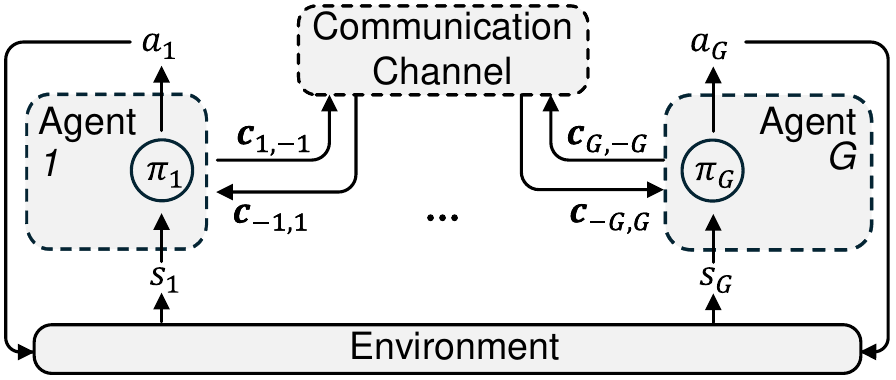}
\caption{System model of a general multi-agent emergent communication system.}
\label{Figure_SANEmergModel}
\end{figure}

\subsection{System Model}
We consider a general emergent communication model for AgentNet systems in which the agents are not given any pre-defined communication languages or protocols, but must autonomously develop a signaling convention or protocol that maps local observations to discrete or continuous representations to facilitate the collaborative optimization of a shared task, as illustrated in Fig. \ref{Figure_SANEmergModel}. 

More formally, we define a general multi-agent emergent communication system as a tuple $\langle {\cal Y},  {\cal G}, {\cal A}, {\cal S}, \bc \rangle$, consisting of the following key elements:  
\begin{itemize} 
\item ${\cal Y}$ is a set of $N$ tasks. Each task can be divided into sub-tasks to be solved collaboratively by a set of agents. % each of which has a specific optimization objective defined by a specific loss function.  
    
\item ${\cal G}$ is the set of agents. Each agent $i \in {\cal G}$ is an autonomous entity that observes its local state, reasons about how to solve a task, and takes action accordingly. The optimization objective of each agent $i$ is defined by a loss function ${\cal L}_{i} \left( a_i, s_i, \bc_{-i,i} \right)$ where $a_i$ is the action, $s_i$ is the observed state, and $\bc_{-i,i}$ is the communication messages sent by other agents to agent $i$.

\item ${\cal A}$ is the action space that defines the set of actions that can be taken by the agents. We can write ${\cal A} = \{ {\cal A}_i \}_{i \in {\cal G}}$ where ${\cal A}^i$ is the set of actions that can be performed by agent $i$, for $a_i \in {\cal A}^i$. 
    
\item ${\cal S}$ is the state space that defines the states that can be observed by the agents. We have ${\cS} = \{ {\cal S}_i \}_{i \in {\cal G}}$ where ${\cal S}^i$ is the set of states observed by an individual agent $i$, for $s_i \in {\cal S}^i$. 

\item $\bc$ is the emergent communication messages learned by a set of agents to exchange information with each other to help with their collaborative task solving. 
    Let $\bc_{i,j}$ be the learned messages sent from agent $i$ to $j$. We also write $\bc_{i,-i} = \{ c_{i,j} \}_{j\in {\cal G}, i \neq j}$ as the messages learned by agent $i$ to communication with all the other agents. 
    
\item ${\pi}_i$ is the decision making function for each agent that maps the current observation and the information received from other agents to its action and messages generated for other agents, i.e., we can write ${\pi}_i: \langle {s_i, \bc_{-i,i}} \rangle \rightarrow \langle a_i, \bc_{i,-i} \rangle$. In this paper, we consider a deep-learning-based emergent communication system in which the decision-making function corresponds to a deep learning model parameterized by $\langle \bphi_i, \bpsi_i \rangle$ for state-to-action mapping $\pi^A_{\bphi_i}: \langle s_i, \bc_{-i,i}\rangle \rightarrow a_i$ and state-and-messages-to-message mappings $\pi^C_{\bpsi_i}: \langle s_i, \bc_{-i,i}\rangle \rightarrow \bc_{i,-i}$, i.e., we can write ${\pi}_i = \langle {\pi}^A_{\bphi_i}, {\pi}^C_{\bpsi_i} \rangle$. In this case, we can rewrite the loss function of each agent $i$ as ${\cL}_{i} (a_i, s_i, \bc_{-i,i}; \bphi_i, \bpsi_{-i})$. 
\end{itemize}

    %Each task $\tau \in {\cal T}$ defines a specific goal specified by a loss function. Let ${\cal L}$ be the set of loss functions.  % ${\cal L}^{\tau}$. % that can be solved by an agent. %a collaborative set of agents ${\cal G}^{\tau}$. 
    % \item $\bGamma$ is the loss functions of the agents. We can write $\bGamma = \langle \Gamma_{i \in {\cal G}} \rangle$. 
    %More specifically, when a collaborative set of agents ${\cal G}^{\tau}$ is selected to solve task $\tau$, communication messages are learned and generated to support information exchanges among agents. 
 %We can write ${\cal C} = \{ {\cal C}^{\tau} \}_{\tau \in {\cal T}}$ where ${\cal C} = \{ c^{\tau}_{i,j} \}_{i, j\in {\cal G}^{\tau} }$. 
 
The above model is very general and can be extended to investigate the emergent communication messages in a wide variety of multi-agent networking scenarios. For example, our model can be directly extended to the large-scale emergent communication systems considered in \cite{Chaabouni2022EmergComAtScale}, in which a set of ``speaker" agents tries to autonomously learn a communication language to describe a target image object to a set of ``listener" agents. In this case, the agents would include both speakers and listeners, where speakers and listeners are dynamically paired to engage in a ``discrimination game". The general task set ${\cal Y}$ would then correspond to identifying image objects for the listeners, and popular loss functions, such as cross-entropy and policy gradient losses, can be adopted to drive learning of the emergent messages. The decision-making functions of the agents can naturally support the population dynamics adaptation policies, such as the imitation and voting policies. %, where agents refine their internal mappings based on the successful strategies of their peers. This flexibility demonstrates that 
Furthermore, our system model is not only capable of representing signaling games, commonly considered in traditional emergent communication scenarios, but is also general enough to investigate more complex emergent communication scenarios, such as the strategic communication\cite{xiao2022RateDistortion}. % support scaling in datasets, task complexity, and population size required for as will be discussed later in this paper.

%We can also observe that the above formulation is also fundamentally different from the interactive POMDP due to the following reasons. 

%The operation procedure of an AgentNet is described as follows. 

\subsection{Problem Formulation}

In this paper, we investigate the emergence of communication protocols and message structures within a practical AgentNet architecture \cite{XY2026SANet}, in which the communication bandwidth between agents, as well as the computational resources available to each agent for decision-making and message emergence, are under the following constraints: % communication messages and protocols that can be learned in a practical AgentNet system in which the communication between agents is limited by the communication and computational resource constraints. %Also, each agent has a limited computational resource for decision making and generating emergent communication messages. Therefore, in this paper, 
%More formally, we define the following two constraints on communication and computational resources for SANEmerg: 

\noindent{\bf C1) (Communication) bandwidth constraint (B-constraint): } The communication messages learned by each agent to be sent to other agents should meet the bandwidth constraints between agents, i.e., the dimensional size of the learned message sent from one agent to another needs to satisfy % the following constraint: 
  \begin{eqnarray}
    {\rm Dim} \left( \bc_{i,j} \right) &\le& B_{ij}, 
    \label{eq_Bconstraint}
  \end{eqnarray}
  where ${\rm Dim}\left( \bc_{i,j} \right)$ is the dimensional size of message $\bc_{i,j}$ sent from agent $i$ to $j$ and $B_{ij}$ is the maximum dimensional size allowed by the bandwidth constraint between agents $i$ and $j$. 
  
\noindent{\bf C2) (Computational) complexity constraint (C-constraint):} The computational resource required for the emergence of communication should also be limited. In this paper, we adopt the MDL as the main metric for measuring the required computational complexity of the deep learning models adopted by each agent to learn and generate communication messages. We can therefore write the C-constraint for a SANEmerg system as follows:
\begin{eqnarray}
%\nonumber % Remove numbering (before each equation)
I \left( s_{i}; \bc_{i,j} \right) &\le& C_{ij}, \forall i, j \in {\cal G}, 
\label{eq_Cconstraint}
\end{eqnarray}
where $s_{i}$ is the local state information observed by agent $i$ and $C_{ij}$ is the maximum computational complexity allow by agent $i$ to generate messages for communicating with agent $j$. % and the $c_{i,j}$ is the 

The main objective of emergent communication is to allow a set of agents to autonomously learn communication protocols to optimize their collaborative goal, i.e., suppose a set of agents $\cal G$ is selected to collaboratively solve a task. We can therefore write the optimization problem as: 
\begin{eqnarray}
&&{\bf P1:}\;\;\; \min_{\bphi, \bpsi}  \sum_{i \in {\cG}} {\cL}_{i} (a_i, s_i, \bc_{-i,i}; \phi_i, \bpsi_{-i}) \\
&&\;\;\;\;\;\;\;\;\;\; \mbox{ s.t. } \; \mbox{Constraints in } (\ref{eq_Bconstraint}) \mbox{ and } (\ref{eq_Cconstraint}).
\end{eqnarray}
where $\bphi = \langle \phi_i \rangle_{i\in {\cal G}}$ and $\bpsi = \langle \psi_i \rangle_{i\in {\cal G}}$ are the model parameters learned by the set of agents for making decisions and emergent communications. 

We can observe that the problem {\bf P1} is different from the previous AgentNet solutions, which focus on decentralized solutions\cite{XY2026SANet}, such as Pareto-optimal solutions. This is because it is known that decentralized optimization with no or limited information among agents often results in various equilibrium solutions, e.g., Pareto-optimal solutions, which are generally non-optimal. In SANEmerg, however, we investigate the potential benefits that can be achieved by allowing communication among agents, where, in the ideal case, the communication and interaction among agents should result in a globally optimal solution. 

It can be observed that the problem {\bf P1} is generally difficult to solve due to the following reasons. First, bandwidth availability between agents can be dynamic, and most existing emergent communication solutions either ignore the impact of communication bandwidth constraints or require model retraining when the bandwidth between agents changes. Second, each agent generally has a limited computational power and can only support communication emergence with a finite computational complexity. Currently, there is still a lack of a unified emergent communication that can support agents with computationally bounded intelligence. Finally, most existing solutions in AI-based communication require separately training dedicated coding schemes and networking protocols for every task and environmental deployment. These solutions then often focus only on the performance improvement when the pre-trained models are deployed, while ignoring the computational and communication overhead associated with the model training phase.  % with both communication and computational constraints at each agent 

%The above  at the agents make the emergent communication problem extremely challenging to solve, especially for the AI model-based emergent communication. First, the previous solution suggests that for each given bandwidth and computational constraints, the emergent communication protocol should be re-trained, which is generally impossible for most practical networking scenarios where the bandwidth can be dynamically changed. Second, 

%In the remainder of this paper, we investigate the emergence of communication among agents under various communication and computational constraints. 

\section{SANEmerg Architecture} % and Emergent Communication}
\label{Section_SANEmergArchitecture}
\subsection{SANEmerg Architecture}
We propose SANEmerg, an emergent communication framework that supports the learning and generation of communication messages for dynamic, goal-driven collaboration among autonomous agents. The workflow of the SANEmerg consists of three sequential phases, which allow the system to continuously and autonomously adapt to heterogeneous user requests and environmental conditions, as described as follows: (1) {\em Semantic-aware task/goal cognition:} Each agent keeps track of the localized user inputs or service requests. Upon detecting a stimulus, e.g., identifies the user's underlying goal, such as ``optimize the quality-of-experience for an immersive communication service" from an application interface, the receiving agent initiates the semantic-aware task/goal cognition phase, % The initial stage of the SANEmerg focuses on 
translating the raw input, e.g., the high-level semantic intent identified from the user's input, into standardized, actionable task representations. The agent evaluates its local computational and functional capabilities against the semantic requirements of the identified task and then determines whether the task objective can be fulfilled independently or requires assistance from other agents. If multi-agent collaboration is required to fulfill the task, e.g., the user goal identified by an application-layer agent requires the coordinated optimization of the network routing configuration controlled by the network-layer agents, and the spectrum and transmit power allocation managed by the physical-layer agents, the workflow transitions to the subsequent phase.
%
%set of objectives that are actionable for the agents. More specifically, an agent controller, powered by a Large Language Model (LLM), first performs semantic reasoning based on the user's input, e.g., identifies the user's underlying goal, such as ``optimize the quality-of-experience for an immersive communication service". It will then break down the identified semantic goal into a set of sub-tasks. For example, in the multi-agent cross-layer optimization case considered in \cite{XY2026SANet}, the identified user goal is decomposed into a set of sub-tasks, including the feature extraction (Application Layer), routing selection (Network Layer), and spectrum and transmit power allocation (Physical Layer). 
%
(2) {\em Agent selection and communication emergence: } To facilitate the multi-agent collaboration, the initiating agent first decomposes the primary semantic goal into a structured set of discrete sub-tasks that are actionable for a set of specialized agents. The assignment of sub-tasks and agents can be orchestrated in conjunction with an agent controller that maintains a comprehensive registry of the skill sets or functional capabilities of all agents within the network. Furthermore, this controller maintains a repository of previously learned emergent communication protocols, systematically indexed by task types, environmental contexts, and combinations of agents. If it identifies a match for the current task, environment, and agent combination, the previously established protocol will be retrieved and deployed.
Otherwise, if the requested task, environments, or agent collaboration is novel, it triggers the communication emergence procedure. During this stage, the selected agents will establish a new language or communication protocol for interaction and collaboration under the current environment from scratch. This newly developed emergent protocol and messages % will be optimized for the specific sub-tasks and environmental states of the collaborative agents, e.g., the application-interface and supported configurations of application-layer agents, and the channel-state-information (CSI) or signal-to-noise ratios of the physical-layer agents. The emergent messages 
can be discrete or continuous vectors that represent the state information most relevant to the shared goal. Agents exchange their learned representations via communication channels, and through iterative learning and updating, the agents learn which ``signals" lead to successful collaborative outcomes, effectively evolving toward the emergence of a communication protocol that is optimized for the task and the dynamics of the environmental states of the collaborative agents, e.g., the application-interface and supported configurations of application-layer agents, and the channel-state-information (CSI) or signal-to-noise ratios of the physical-layer agents.
%
%
%and is structurally bounded by the prevailing features and physical constraints of the environment. %undergo a joint optimization process to evolve a bespoke signaling convention or ``language" from scratch. 
%
%This newly developed emergent protocol is strictly optimized for the specific sub-tasks and is structurally bounded by the prevailing physical constraints of the environment.
%
%Following the decomposition of the user's semantic goal, the decomposed sub-tasks will be mapped to a set of specialized agents, e.g., agents in application, network, and physical layers. These agents need to establish a common language for interaction and collaboration. Each agent autonomously develops a ``signaling convention" by mapping its locally observed environmental states (e.g., the application-interface and supported configurations of application-layer agents, and the channel-state-information (CSI) or signal-to-noise ratios of the physical-layer agents) into message outputs. These are discrete or continuous vectors that represent the state information most relevant to the shared goal. Agents exchange their learned representations via communication channels, and through iterative learning and updating, the agents learn which ``signals" lead to successful collaborative outcomes, effectively evolving toward the emergence of a communication protocol that is optimized for the task and the dynamics of the communication channels. 
%
(3) {\em Task execution and decision-making: } Following the successful convergence of the emergent communication protocol, the multi-agent collaboration transitions from the learning phase to real-time operational deployment. In this case, the selected agents will exchange learned communication messages, coordinate their localized policies, and execute their sub-tasks synchronously to fulfill the overarching user objective. This collaboration will continue until the identified task is resolved. 

The above workflow will be automatically reinitiated whenever any agent in the network detects a new user or service request. % from its local user or interfacing service.

%SANEmerg workflow ensures that the multi-agent collection remains highly responsive to both volatile user demands and fluctuating physical resources. 

%Once all the agents have finished learning and a communication protocol has emerged among agents for the collaboration, the agents will start to exchange the learned communication messages in real-time, and each agent should then interpret the incoming signals from others to gain a ``global" understanding of the environmental state. Based on the collective intelligence, agents can independently execute local actions, and the decentralized actions converge to satisfy the original semantic goal identified in step (1). 

%The operational procedures are illustrated in Fig. ?. 

%Existing agentic AI solutions, such as Manus and OpenClaw, have already shown promising results in semantic goal recognition and sub-task separation. Therefore, in this paper, we focus on agentic execution and assume a certain goal has already been identified and a set of agents has already been chosen. We will then focus on investigating the emergence of communication among collaborative agents. The joint optimization of the goal identification, sub-task separation, and emergent communication will be left for our future work. 

\subsection{Emergent Communication}
%As mentioned earlier, problem P1 is in generally difficult to solve due to the . Therefore, before we present the detailed approach to learn the emergent communication protocols among agents, let us first convert the problem P1 into a

To address these challenges of problem {\bf P1}, we would like to design an emergent communication framework that meets the semantic goal of the users with the physical constraints of the network, i.e., an ideal framework should possess three properties: 

\noindent{\bf (1) Bandwidth adaptable:} Since the communication bandwidth between agents is limited and can be dynamically changing, the emergent communication protocol is therefore required to support dynamic adaptation according to available bandwidths without requiring model re-training. 

\noindent{\bf (2) Computational complexity-limited: } Since agents typically operate under finite power and processing limits, they can only support a restricted amount of computational overhead. The emergent communication protocol must be deployable on agents with fixed processing budgets. %Prior research has indicated that the quality and richness of emergent communication are profoundly influenced by these complexity bounds, yet a unified framework for investigating this relationship across varying constraints is still largely absent. 

\noindent{\bf (3) Scalable and efficient to learn and deploy: } Most previously proposed emergent communication solutions required the development of delicate modules for coding, feature extraction, and recovery. It is critical to develop a novel solution that minimizes the time and resources required to build these delicate modules for practical deployment, as the overhead of the learning phase can often offset the gains in communication efficiency.

In the rest of this section, we introduce two key components of SANEmerg that solve the above challenges: an importance-filter for bandwidth adaptation and a complexity-regularizer to enforce computational complexity constraints. 

\subsubsection{Importance-filter for bandwidth adaptation} 
Let us first introduce a learnable importance-filtering mechanism that prioritizes the communication of task-relevant information when the bandwidth between agents is limited. More specifically, we follow the same approach as in \cite{XY2026SANet} to develop a learnable importance-filter that assigns dynamic weights to each output message dimension based on its contribution to the global optimization objective. Under high bandwidth availability, the mechanism enables the exchange of high-dimensional information, facilitating near-lossless observation exchange among agents for accurate decision-making. As bandwidth becomes limited, the importance-filter selectively discards dimensions with low task relevance to ensure that critical semantic information is preserved, thereby maintaining robust system performance and minimizing the adverse effects of communication bottlenecks on agents' collaborative actions.

More formally, suppose the message learned by agent $i$ for communicating with agent $j$ has dimensional size  ${\rm Dim} (\bc_{i,j})$. Each agent $i$ needs then learn an importance score $\kappa_{i,d}$ to quantify the contribution of each dimension of its output message for $d \in \{1, \ldots, {\rm Dim} (\bc_{i,j})\}$ to achieve the task goal by incorporating an exponential regularizer $\exp \left( \sum^{{\rm Dim} (\bc_{i,j})}_{d=1} \kappa_{i,d} \right)$ into the sub-task loss function of the agent $i$, i.e., the loss function of the optimization objective is then given by
\begin{eqnarray}
\cL'_i = {\cL}_{i} (a_i, s_i, \bc_{-i,i}) + \epsilon_i {\cL}^B_{i} \left( \bkappa_i \right),
\end{eqnarray}
where $\epsilon_i$ is the coefficient of the importance-filter regularizer. It determines how the filter is weighted against the task loss during the training phase. ${\cL}^B_{i}$ is given by,
\begin{eqnarray}
{\cL}^B_{i} \left( \bkappa_i \right) &=& \exp \left( \sum^{{\rm Dim} (c_{i,j})}_{d=1} \kappa_{i,d} \right).
\label{eq_LossofBconstraint}
\end{eqnarray}

Jointly optimizing the task loss and the regularizer encourages each agent to identify a minimal yet highly informative subset of dimensions in its learned message. If a specific dimension $d$ does not contribute significantly to minimizing the sub-task loss $\mathcal{L}_{i}$, the optimization process will reduce the value of $\kappa_{i,d}$ to minimize the penalty, effectively ``turning off" or de-emphasizing that dimension of the message vector. 

As shown in the experimental results later, our proposed importance-filter significantly reduces the impact of bandwidth limitations on the collaborative decision-making accuracy, especially compared to other existing solutions.

\subsubsection{complexity-regularizer for complexity-limited emergence} 

As mentioned earlier, directly optimizing the mutual information in C-constraint is notoriously difficult in deep learning frameworks because the marginal distribution of the emergent messages is generally intractable. To address this challenge of the C-constraint, let us first introduce a variational upper bound to approximate the C-constraint.

\begin{theorem}(Variational Upper Bound of C-constraint)
\label{Theorem_VariationalBoundCconst}
Let $p(c_{i,j}|s_i)$ be the conditional probability distribution of emergent messages $c_{i,j}$ given the observed state $s_i$ for agent $i$ and $q(c_{i,j})$ be an arbitrary prior distribution in the message space. The mutual information $I(s_i; c)$ is upper-bounded by the expected Kullback-Leibler (KL) divergence between the encoder output and the prior:
\begin{eqnarray}
I(s_i; c_{i,j}) \le \mathbb{E}_{s_i \sim p(s_i)} \left[ D_{KL} \left( p(c_{i,j}|s_i) \parallel q(c_{i,j}) \right) \right],
\label{eq_CconstraintApprox}
\end{eqnarray}
\end{theorem}

\begin{IEEEproof}
By definition, $I(s_i; c_{i,j}) = \mathbb{E}_{s_i,c_{i,j}} \left[ \log \frac{p(c_{i,j}|s_i)}{p(c_{i,j})} \right]$. Since $D_{KL}(p(c_{i,j}) \parallel q(c_{i,j})) \ge 0$, it then follows that $\mathbb{E}_{s_i} [ \log p(c_{i,j}) ] \ge \mathbb{E}_{s_i} [ \log q(c_{i,j}) ]$. Substituting this inequality into the mutual information formula yields: $I(s_i; c_{i,j}) \le \int p(s_i) \int p(c_{i,j}|s_i) \log \frac{p(c_{i,j}|s_i)}{q(c)} d c_{i,j} d s_i$ which can then be simplified to the expected KL divergence term in the right-hand-side of (\ref{eq_CconstraintApprox}). 
\end{IEEEproof}

By choosing a tractable prior, e.g., a standard multivariate Gaussian, for $q(c_{i,j})$, we can effectively penalize the complexity of the message $c_{i,j}$ during training. As will be shown in the experimental results section, the above upper bound is quite tight in approximating the MDL of the computational complexity. 

Building on the variational upper bound established in Theorem \ref{Theorem_VariationalBoundCconst}, we can mathematically formalize how communication messages emerge from computational processes. In this framework, the locally observed environmental state $s_i$ contains a vast amount of entropy, much of which is structurally indistinguishable from the noise to a computationally bounded observer. The act of ``learning" or ``emergence" is characterized by the feature extraction distribution $p_{\pi_{\bphi_i}}(c_{i,j}|s_i)$, which maps the high-dimensional state information $s_i$ to a low-dimensional feature $c_{i,j}$ to be sent from agent $i$ to agent $j$. From Theorem \ref{Theorem_VariationalBoundCconst}, we can observe that the complexity of this emerged information is bounded by the KL divergence between the learned representation and a priori knowledge, i.e., $\mathbb{E}_{s_i \sim p(s_i)} \left[ D_{KL} \left( p_{\pi_{\bpsi_i}}(c_{i,j}|s_i) \parallel q(\bc_{i,j}) \right) \right]$. This term represents the MDL that the system can effectively utilize. In particular, when the computational resources of each agent are limited by constraint $C_{ij}$ defined in (\ref{eq_Cconstraint}), the agent $i$ is forced to discard non-essential data and retain only the most critical structural regularities. This further implies that the semantic information is not a static property residing in the observed state $s_i$, but is a ``computational product" emerged from the optimization of the variational bound. The volume of the emerged information is directly proportional to the computational complexity of model $\pi_{\bphi_i}$. By increasing the computational complexity allowed for the agent $i$, i.e., a larger value of $C_{ij}$, the agent can resolve more intricate patterns in its observed state $s_i$, effectively ``creating" more learnable semantic content. Thus, the variational upper bound is not simply a measure of the information transmission limit; it quantifies the threshold at which the observed state is transformed into actionable intelligence through bounded computation.

By replacing the B- and C-constraints with (\ref{eq_LossofBconstraint}) and (\ref{eq_CconstraintApprox}), we can rewrite the problem {\bf P1} into the following Lagrangian form:
\begin{eqnarray}
\lefteqn{ \mbox{\bf P2: } \;\; \min_{\bphi, \bpsi, \bkappa}  \sum_{i \in {\cG}} \left[ {\cL}_{i} (a_i, s_i, \bc_{-i,i}; \phi_i, \bpsi_{-i}) + \right. } \nonumber \\ 
&&\;\;\;\;\;\;\;\;\;\;\;\;\;\;\;\;\;\;\;\;\;  \left. \epsilon_i {\cL}^B_{i} \left( \bkappa_i \right) + \epsilon'_i {\cL}^C_{i} \left( \bpsi_i \right)  \right],
\end{eqnarray}
where $\epsilon_i$ and $\epsilon'_i$ are weighting coefficients of the importance-filter and complexity regularization. They control the trade-off between task accuracy and communication-computational resource efficiency. ${\cL}^B_{i} \left( \bkappa_i \right)$ is given in (\ref{eq_LossofBconstraint}) and ${\cL}^C_{i} \left( \psi_i \right)$ is given by 
\begin{eqnarray}
{\cL}^C_{i} \left( \psi_i \right) = \mathbb{E}_{s_i \sim p(s_i)} \left[ D_{KL} \left( p_{\pi_{\bpsi_i}}(c_{i,j}|s_i) \parallel q(c_{i,j}) \right) \right]. 
\end{eqnarray}

We can observe that the above framework offers several unique advantages for emergent communication in AgentNet systems. First, by consolidating the agent's task objectives and generated message protocols into a unified training objective, this framework eliminates the need for a separate development phase for communication protocol modules. This integration significantly enhances scalability and reduces the training latency and overhead traditionally associated with isolated communication modules. Second, the inclusion of the importance-filter enables a ``train once, deploy anywhere" capability; the protocol can adapt to fluctuating bandwidth constraints in real-time by simply masking dimensions with low importance scores, ensuring robust performance without retraining. Finally, the MDL bound-based regularizer provides a theoretical foundation for investigating how computational complexity shapes the emergence of communication. In this paradigm, the ``meaning" captured by the system is not a static reflection of the source data observed by the agents, but a dynamic product of the resources invested in the emerging process. Greater computational effort, i.e., higher computational complexity, facilitates the discovery of deeper structural regularities, expanding the description length that agents can effectively utilize in the message emergence. %In this way, semantic information can be treated as the outcome of the interplay between raw observations and the processing complexity of agents.

% \section{Theoretical Analysis}

% Generalization: 

% Adaptation and Transferability: 

% Robustness: 

%\subsection{Communication Emergence with both B-Constraint}

%\subsection{Joint Decision Making and Emergent Communication Algorithm}

\section{Prototype and Experimental Results}
\label{Section_Prototype}

\begin{figure}[t]
\centering
\includegraphics[width=0.9\linewidth]{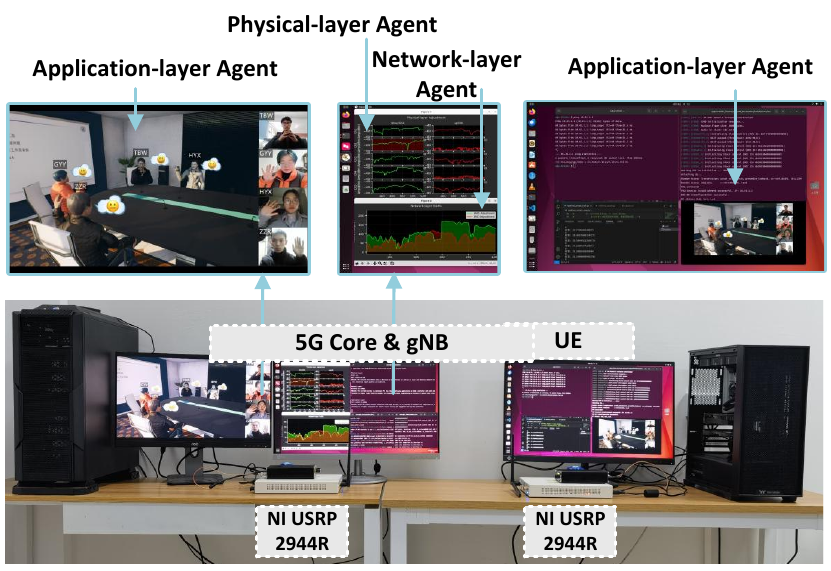}
\caption{An AgentNet prototype.}
\label{Figure_Prototype}
\vspace{-0.2in}
\end{figure}

\subsection{Prototype and Experiment Setup}
We developed an AgentNet prototype, consisting of user equipments (UEs), a gNodeB (gNB), and a 5G core (5GC) network, based on an open-source RAN and softwareized 5G core network platform, as shown in Fig. \ref{Figure_Prototype}.   
We investigate the communication emergence between an application-layer agent (aAgent) and a physical-layer agent (pAgent), in which the aAgent acts as the ``source of user's semantic intent" and the pAgent acts as the ``physical-layer resource actuator", in which the aAgent observes the user behavior and detects the semantic requests (e.g., ``I need a low-latency video stream for an online gaming" or ``I want to watch high-resolution video for long-video streaming") and will learn to evolve a communication language that translates the abstract user intent into communication messages sent to the pAgent. The pAgent will then learn to determine the most efficient physical resource block configuration to meet the user's intent based on the signal it receives from the aAgent. The datasets and models of both agents are described as follows: 
%\begin{itemize}
%\item[(1)] {

\noindent{\bf Application-layer agent (aAgent): } We adopt a publicly available data traffic dataset collected from commercial smartphones across five diverse mobile application categories\cite{choi2023ml}. These include live streaming (NaverNow), long-video streaming (Netflix), online conferencing (Zoom), online gaming (Battlegrounds), and game streaming (GeForce Now). To reflect the various requirements of these services, we evaluate performance using application-specific Quality-of-Experience (QoE) metrics. Specifically, we prioritize latency-based metrics for the online gaming and game streaming applications, while focusing on average throughput for the video streaming and online conferencing applications. The aAgent observes and predicts the types and real-time dynamics of the application traffic and learns to send optimized communication signals to the pAgent. The pAgent is expected to dynamically adjust and allocate resources to ensure that the underlying physical-layer configurations can support the application's real-time data traffic demands.

%developed an immersive 3D video streaming application, consisting of client applications installed at two UEs, each of which uploads its local 2D video of a human user to a server application installed at an edge server installed at the 5GC, which converts the UEs' 2D videos into corresponding 3D avatars in a metaverse environment. The 3D video generated by the server will be streamed back to the UEs to offer an immersive 3D conference experience. The and autonomously decides the optimal parameters of applications that stream the 3D video from the 5GC to the UEs. It also translates these intents into emergent signals sent to the pAgent. 

%\item[(2)] 
\noindent{\bf Physical-layer agent (pAgent): } The pAgent observes the physical-layer environment, e.g., CSIs of the 5G NR frequency bands between the UEs and gNB, and decides the most efficient physical resource block configuration to meet the user's intent based on the signal received from the aAgent. We feed the mobile traffic data packets generated by different mobile apps into the UE of our developed AgentNet prototype and evaluate the impact of different physical-layer resource configurations on the user's QoE. The reward is calculated based on the user's semantic goal fulfillment. If the pAgent allocates too few physical-layer resources, the user request fails (with a low reward). If it allocates too many, it wastes network energy/efficiency (lower reward). Optimal reward is achieved when the allocated physical-layer resource matches the application requirements. % intersection of task success and resource parsimony.
%\end{itemize}

 %. We can then evaluate the emergence of communication between pAgent and aAgent as they collaborate to  that can handle a variety of user requests   

\subsection{Experimental Results}

To evaluate the impact of emergent communication under various physical resource constraints, we compare the performance of the following schemes on our prototype across communication and computational complexity constraints. 

% \noindent{\bf Scheme 1 (No-EC):} This scheme represents traditional communication paradigms where a human-designed mapping exists between application-layer requirements and physical-layer resource blocks. The aAgent transmits pre-defined command IDs (e.g., ``High-Priority Video," ``Background Data"), and the pAgent executes a static resource allocation strategy for each ID. This baseline serves to demonstrate the lower bound performance and illustrates the limitations of rigid, non-adaptive protocols in handling complex, dynamic user intents. 

\noindent{\bf EC-SOTA:} This scheme follows the methodology developed in \cite{Lin2021EmergCommAutoEncoder}, in which, in addition to training their decision-making models, each agent also trains a dedicated autoencoder model to learn latent vectors for communication without explicit constraints on message sizes or model complexity. This scheme serves as the benchmark for the emergent communication. 

\noindent{\bf SANEmerg-IF:} This scheme employs the importance-filter described in Section \ref{Section_SANEmergArchitecture} to adapt to various bandwidth constraints but utilizes large unregularized Transformer-based models for both aAgent and pAgent. This scheme isolates the impact of the MDL-based complexity-regularizer on reducing computational overhead and ensuring the framework's feasibility for deployment under dynamic bandwidth constraints. 

\noindent{\bf SANEmerg-CR:} This scheme utilizes the MDL-based complexity-regularizer to optimize agent neural architectures but employs a standard, fixed-bandwidth communication link between the aAgent and pAgent. Comparing this to the full SANEmerg framework reveals the gain in communication efficiency and resilience provided by the Importance-Filter under fluctuating channel conditions.

\noindent{\bf SANEmerg-IF-CR:} This scheme incorporates both an importance-filter and the MDL-based complexity-regularizer to enable resource-efficient communication emergence.

% \subsubsection{Bandwidth limitations}

% \subsubsection{Complexity limitations}

\begin{figure}[t]
	\centering
	\includegraphics[width=0.7\linewidth]{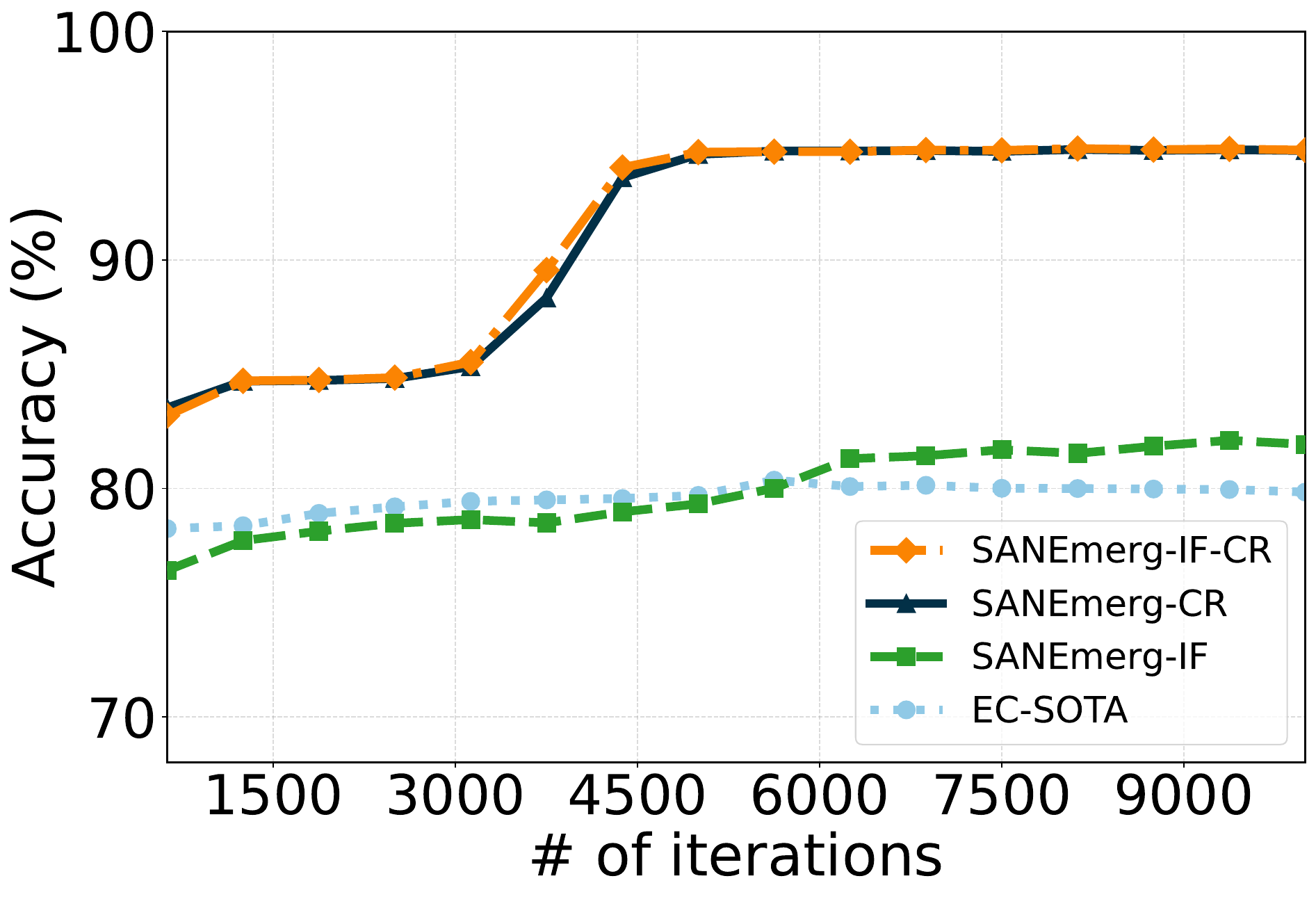}
	\caption{Comparison of convergence performance of the aAgent under different numbers of iterations. }
    \label{Fig_Comparison_of_convergence_performance}
\end{figure}

We first compare the convergence of the aAgent's accuracy when trained by different schemes across five distinct mobile application scenarios under different numbers of iterations in Fig. \ref{Fig_Comparison_of_convergence_performance}. We observe that the full SANEmerg implementation, i.e., SANEmerg-IF-CR, consistently achieves the highest accuracy across all tested applications, ranging from 93\% to 96\%. In comparison, the EC-SOTA significantly underperforms, with an accuracy drop of approximately 10-12\% relative to the full SANEmerg model. When only the importance-filter is utilized, i.e., SANEmerg-IF, we observe a slight improvement over the EC-SOTA when the bandwidth constraint is not very tight. This confirms that, even when the communication bandwidth is not a bottleneck, filtering out non-essential state noise is also helpful for the aAgent to focus on transmitting the meaningful semantic intent to the pAgent. %This performance gap highlights the limitations of standard emergent communication methods that lack explicit semantic grounding, as they tend to overfit to specific training noise rather than to capture invariant, task-relevant features. 
Furthermore, using only the complexity-regularizer, i.e., SANEmerg-CR, yields better results than both the EC-SOTA baseline and SANEmerg-IF. %, though it is generally slightly lower than SANEmerg-IF. 
This is because the complexity-regularizer acts as a structural constraint that prevents the emergent messages from becoming overly complex or high-dimensional. This ``compression" effect forces the aAgent to develop more efficient and generalized signaling schemes.

\begin{figure}[t]
    \centering
    \begin{minipage}[t]{0.48\linewidth}
    \includegraphics[width=1\textwidth]{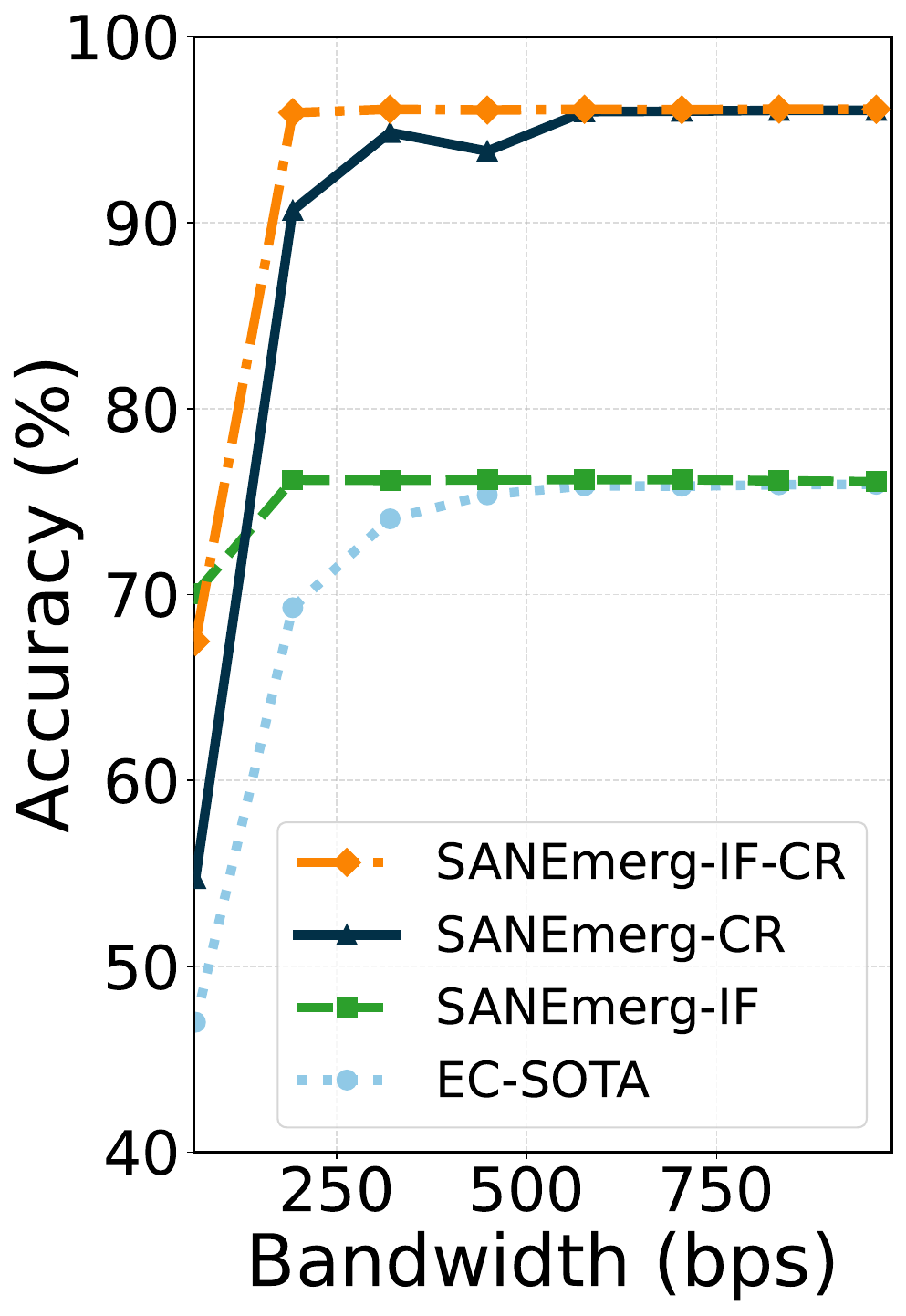}
    \captionsetup{labelformat=empty}
    \vspace{-0.1in}
    \caption*{(a)}
    \label{fig:sub_a}
    \end{minipage}
    \hfill
    \begin{minipage}[t]{0.48\linewidth}
    \includegraphics[width=1\textwidth]{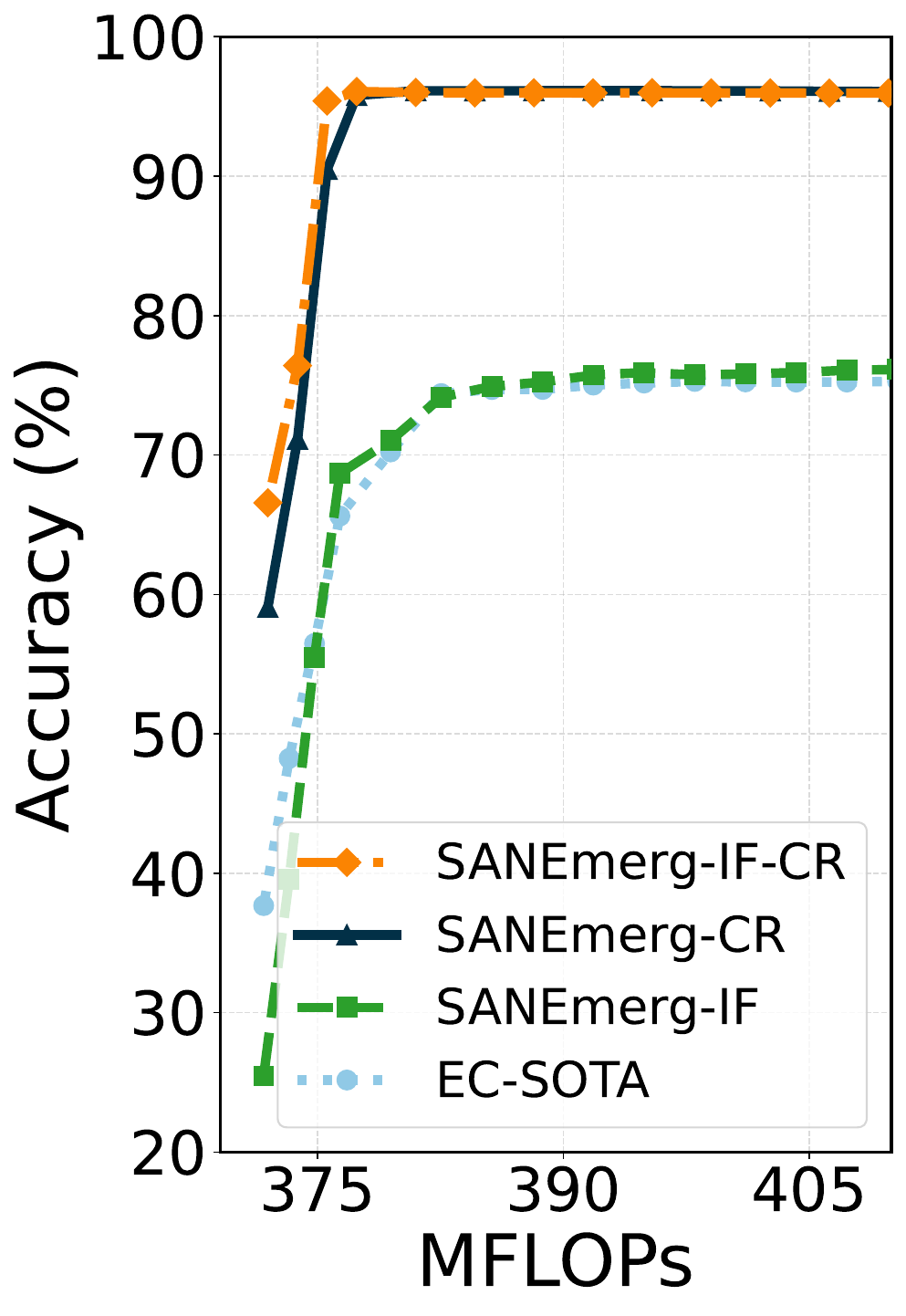}
    \captionsetup{labelformat=empty}
    \vspace{-0.1in}
    \caption*{(b)}
    \label{fig:sub_b}
    \end{minipage}
    \vspace{-0.1in}
    \caption{\small \blu{Inference accuracy under (a) bandwidth constraints, and (b) different computational complexity constraints.}}
    \label{Figure_BandwidthandComplexity}
    \vspace{-0.1in}
\end{figure}

We then evaluate the accuracy of the SANEmerg framework under various physical-layer constraints. In particular, in Fig. \ref{Figure_BandwidthandComplexity}(a), we evaluate the impact of communication bandwidth limitations on model accuracy. We can observe that the EC-SOTA baseline requires nearly 750 bps of bandwidth to reach its maximum of approximately 76\%. In contrast, the SANEmerg-IF-CR scheme achieves much higher efficiency, reaching a peak accuracy of approximately 96\% even with a bandwidth of only 250 bps, resulting in an improvement ratio of approximately 39.1\%. We also investigate the relationship between computational complexity (measured in MFLOPs) and inference accuracy of the aAgent in Fig. \ref{Figure_BandwidthandComplexity}(b). We observe that in the computationally limited scenario, e.g., with only 375 MFLOPs, the EC-SOTA baseline achieves approximately 57\% inference accuracy. The SANEmerg-IF-CR scheme, however, achieves ~95\% accuracy, resulting in a 66.6\% improvement ratio. %We can also observe that  t
While the Importance-Filter alone, i.e., SANEmerg-IF, achieves maximum accuracy faster than the EC-SOTA baseline, both schemes are capped at the same accuracy level. The addition of the complexity-regularizer, i.e., SANEmerg-CR, is the critical factor that unlocks the high-accuracy regime. This confirms that while the Importance-Filter identifies ``what" is important, the complexity-regularizer optimizes ``how" that information is compressed. %, allowing the aAgent to communicate complex intents to the pAgent with extremely tight computational complexity constraints. 

\section{Conclusion}
\label{Section_Conclusion}
%This paper proposes SANet, a novel semantic-aware AgentNet architecture that autonomously identifies the user's semantic goal and autonomously divides the identified goal into different subtasks for different agents. We introduce a novel functional entity, the agent controller, employed with a novel dynamic weighting-based conflict-resolving mechanism. We develop a hardware prototype, and our experimental result suggests that SANet significantly improves the performance of multi-agent networking systems, even with objective-conflicting agents. % under various dynamic environments. 

This paper has proposed SANEmerg, a novel multi-agent emergent communication framework tailored for semantic-aware AgentNet systems. By integrating a bandwidth-adaptable importance-filter and an MDL-based complexity-regularizer, SANEmerg enables autonomous agents to develop efficient, semantic-aware signaling protocols that prioritize transmitting the high-contribution message dimensions while adhering to strict computational complexity and communication bandwidth limitations. Experimental validation based on an AgentNet hardware prototype suggests that SANEmerg achieves significant performance improvements over state-of-the-art solutions, while reducing required communication bandwidth and computational complexity.  %We investigate how communication and signaling protocols can emerge among collaborative agents with computationally bounded intelligence under stringent bandwidth constraints between agents. We propose SANEmerg, a novel framework designed to facilitate the emergence of communication for collaborative task fulfillment while adhering to the physical limits of AgentNet. SANEmerg incorporates a bandwidth-adaptable importance filter that dynamically prioritizes the transmission of higher-contribution message dimensions, ensuring robust performance in bandwidth-limited environments. Furthermore, SANEmerg integrates a complexity-regularizer grounded in the Minimum Description Length (MDL) principle to facilitate computationally-bounded signaling emergence. 

\section*{Acknowledgment}
This work was supported by the National Natural Science Foundation of China under grant 62525109 and the Mobile Information Network National Science and Technology Key Project under grant 2024ZD1300700.

\bibliographystyle{IEEEtran}
\bibliography{DeepLearningRef}

\end{document}